\documentclass[conference]{IEEEtran}
\IEEEoverridecommandlockouts
\usepackage{cite}
\usepackage{amsmath,amssymb,amsfonts}
\usepackage{graphicx}
\usepackage{textcomp}
\usepackage{xcolor}
\usepackage{algpseudocode}
\usepackage{algorithm}
\def\BibTeX{{\rm B\kern-.05em{\sc i\kern-.025em b}\kern-.08em
    T\kern-.1667em\lower.7ex\hbox{E}\kern-.125emX}}
    
\usepackage{amsmath,amsthm}
\usepackage{color}
\usepackage{bbm,bm}
\usepackage{pgfplots}

\usepackage{tikz}
\usetikzlibrary{matrix,positioning,fit,calc,shapes}
\pgfplotsset{compat=newest}
\usepackage{graphicx}
\usepackage{color}
\usepackage{epsfig, amsmath, amsthm, amssymb,latexsym,amsmath,amsbsy, rotating}
\usepackage{amsfonts}
\usepackage{tikz}

\newtheorem{theorem}{Theorem}
\newtheorem{lemma}[theorem]{Lemma}

\theoremstyle{definition}

\newtheorem{example}{Example}

\newcommand{\CLASSINPUTtoptextmargin}{2cm}

\usepackage{hyperref}
\usepackage{xcolor}
\hypersetup{
    colorlinks,
    linkcolor={black},
    citecolor={black},
    urlcolor={black}
}

\DeclareFontFamily{U}{futm}{}
\DeclareFontShape{U}{futm}{m}{n}{
  <-> s * [.92] fourier-bb
  }{}
\DeclareSymbolFont{Ufutm}{U}{futm}{m}{n}
\DeclareSymbolFontAlphabet{\mathbb}{Ufutm}

\begin{document}

\title{On the Semi-supervised Expectation Maximization}

\author{\IEEEauthorblockN{Erixhen Sula}
\IEEEauthorblockA{\textit{The Department of Electrical Engineering} \\
\textit{and Computer Science} \\
\textit{MIT}\\
Massachusetts, USA \\
esula@mit.edu}
\and
\IEEEauthorblockN{Lizhong Zheng}
\IEEEauthorblockA{\textit{The Department of Electrical Engineering} \\
\textit{and Computer Science} \\
\textit{MIT}\\
Massachusetts, USA \\
lizhong@mit.edu}
}

\maketitle

\begin{abstract}
The Expectation Maximization (EM) algorithm is widely used as an iterative modification to maximum likelihood estimation when the data is incomplete. We focus on a semi-supervised EM to learn the model from labeled samples besides the unlabeled ones. By contrast to existing work in semi-supervised learning that focuses mainly on performance, we focus on the impact of the labeled samples on convergence. Our study is limited to population EM or EM with unlimited data. The analysis shows the impact of the labeled samples on the convergence rate for the exponential family mixture model and Gaussian mixture models. To be precise, we prove local convergence for exponential family mixture models where we initialize EM within the neighborhood of the global optimum. Besides, we prove global convergence for Gaussian mixture models by computing the convergence rate.


\end{abstract}

\IEEEpeerreviewmaketitle
\section{Introduction}

The data available to us is mainly incomplete, with missing values or corrupted. Maximum likelihood estimation (MLE) is an essential tool for parametric model estimation, but it is intractable for most incomplete data. This challenge is addressed in part by the Expectation Maximization algorithm that has emerged as an iterative modification to MLE. One major drawback of EM is to output the local optima of the likelihood function. In this regard, there is a rich literature on the convergence of EM to the global optimum of the likelihood function under some mild assumptions \cite{Wu83,Hero-Fessler95,Meng94,Meng-Rubin94,Redner-Walker84,Dasgupta-Schulman07,Xu-Jordan96,chaudhuri09,Xu16,Hardt-Price15,Pearson1894,Wainwright17}. There is little information on how labeled samples affect EM, so primarily we focus on semi-supervised EM. Semi-supervised learning combines a few labeled data with lots of unlabeled data. Labeled data have a cost, whereas unlabeled data are relatively cheap. Depending on your budget, the cost of those data will directly impact the amount of data you want to use. Semi-supervised settings can prove to be essential in such situations. In this paper, we analyze the impact of the labeled data on the convergence rate of EM for different mixture models. To the best of our knowledge, no prior work has analyzed the convergence rate of the EM algorithm for semi-supervised learning. An overview of the contribution is provided below: 



\begin{itemize}
\item The analysis shows the impact of the labeled samples on the convergence rate for the exponential family mixture model. As the exponential family contains non-specified functions, we limit our analysis by initializing the population EM within the neighborhood of the true parameters.

\item For the Gaussian mixture model, we compute the convergence rate for the population EM with labeled samples. The convergence rate is given in terms of the initialized parameters. We prove that for symmetric GMM with two components, the convergence rate does not depend on EM initialization. 
\end{itemize}

The EM algorithm has a long history, yet we will list a summary of the earlier works regarding the convergence rate. Wu \cite{Wu83} is one of the first to establish global convergence of the EM algorithm for unimodal likelihood functions under regularity conditions. However, the likelihood function is frequently multi-modal, and the EM algorithm guarantees only convergence to the local optimum \cite{Hero-Fessler95,Meng94,Meng-Rubin94}. Redner and Walker \cite{Redner-Walker84} showed that for exponential family mixture models, the population EM initialized within the basin of the true parameters of the model converges to the global optimum. For Gaussian mixture models (GMM) in high dimensions, Dasgupta and Schulman \cite{Dasgupta-Schulman07} showed fast global convergence when the clusters of the data generated by GMM are well-separated. Redner and Walker \cite{Redner-Walker84} observed that when the data generated by GMM are not well-separated, EM heuristically converges slowly. Later, Xu and Jordan \cite{Xu-Jordan96} showed theoretically slow convergence of EM for barely separated GMM. For the mixture model of two symmetric Gaussians, Chaudhuri \emph{et al} \cite{chaudhuri09} showed that a hard decision variant of EM known as k-means always converges without any separation assumption. Later, Xu \emph{et al} \cite{Xu16} showed the global convergence of EM for the mixture model of two symmetric Gaussians without any separation assumption. Specifically, for the symmetric mixture of two Gaussian with known variance and equal weights EM algorithm converges to the true mean of the model. However, for two Gaussian mixtures with known variance and unequal weights, the EM algorithm converges to three stationary solutions where only two are global maxima. For finite samples, Hardt and Price \cite{Hardt-Price15} inspected the limits of global convergence up to a constant additive factor for the mixture of two Gaussians by applying the method of moments inspired by Pearson \cite{Pearson1894}. Balakrishnan \emph{et al} \cite{Wainwright17} focused on the required conditions for global convergence of the EM algorithm with associated convergence rates. The assumptions are verified for the symmetric mixture of two Gaussians with a double exponential convergence rate if we initialize the EM algorithm within the basin of global convergence.

\section{Expectation Maximization and Maximum Likelihood}
Let $y_i$ $i=1,\dots,n$ be i.i.d. samples from $p_Y$ and let $(x_j,y_j)$, $j=1,\dots,m$ be i.i.d. samples from $p_{XY}$.
\subsection{Maximum Likelihood Estimation}
The maximum likelihood estimation using all the samples is identical to maximizing log-likelihood
\begin{align}
\arg\max_{\theta} \sum_{j=1}^m \log p_{XY}(x_j,y_j;\theta) + \sum_{i=1}^n \log p_Y(y_i;\theta).
\end{align} 
\subsection{Expectation Maximization}
EM is an iterative algorithm as an alternative to maximum likelihood estimation, by ascending the likelihood function after each iteration. EM at each iteration $t$, performs the E-step
\begin{align} \label{eqn:Qnm}
Q_{n,m}(\theta;\theta^t) &:= \sum_{j=1}^m \log p_{XY}(x_j,y_j;\theta) \nonumber \\
& \quad \quad +\sum_{i=1}^n \sum_{k=1}^K p_{X|Y}(k|y_i;\theta^t) \log p_{XY}(k,y_i;\theta)
\end{align}
and M-step as 
\begin{align} \label{eqn:Mnm}
M_{n,m}(\theta^t)=\arg\max_{\theta}  Q_{n,m}(\theta;\theta^t).
\end{align}
\section{Model}
\subsection{Exponential family mixture model}
The samples $(x_1,y_1),\dots,(x_m,y_m)$ are generated in an i.i.d. fashion from the exponential family distribution
\begin{align} \label{eqn:expfamilyjoint}
p_{XY}(k,y;\theta)&=p_X(k) p_{Y|X}(y|k; \theta) \\
&= \pi_k e^{\theta_k t(y) +\beta(y) -\alpha(\theta_k)}.
\end{align}
The sample $y_1,\dots,y_n$ are generated in an i.i.d fashion from the induced distribution
\begin{align} \label{eqn:expfamilymarg}
p_Y(y;\theta)=\sum_{k=1}^K \pi_k e^{\theta_k t(y) +\beta(y) -\alpha(\theta_k)}.
\end{align}
By plugging (\ref{eqn:expfamilymarg}) and (\ref{eqn:expfamilyjoint}) into (\ref{eqn:Qnm}) we compute $Q_{n,m}(\theta;\theta^t)$ for the exponential family mixture model as follows
\begin{align} \label{eqn:Qnm_comp_expfam}
&Q_{n,m}(\theta;\theta^t) = \nonumber \\
&\frac{1}{n+m} \left( \sum_{j=1}^m \sum_{k=1}^K \mathbb{1}_{\left\{ x_j=k \right\}} \left( \theta_k t(y_j) + \beta(y_j) -\alpha(\theta_k) \right) \right. \nonumber \\
& \quad \quad  \left. +\sum_{i=1}^n \sum_{k=1}^K q(y_i;\theta_k^t) \left( \theta_k t(y_j) + \beta(y_j) -\alpha(\theta_k) \right) \right)
\end{align}
where,
\begin{align}
q(y_i;\theta_k^t) =\frac{ e^{\theta_k^t t(y_i) + \beta(y_i) -\alpha(\theta_k^t)}}{\sum_{k^{\prime}=1}^K e^{\theta_{k^{\prime}}^t t(y_i) + \beta(y_i) -\alpha(\theta_{k^{\prime}}^t)}}.
\end{align}
Plugging the computed $Q_{n,m}(\theta;\theta^t)$ of (\ref{eqn:Qnm_comp_expfam}) into (\ref{eqn:Mnm}) we compute $M_{n,m}(\theta_k^t)$ as follows 
\begin{align} \label{eqn:Exp_fam_update}
M_{n,m}(\theta_k^t)= \beta^{-1} \hspace{-0.5em} \left( \frac{\sum_{j=1}^m \mathbb{1}_{\left\{ x_j=k \right\}}t(y_j) + \sum_{i=1}^n q(y_i;\theta_k^t) t(y_i)}{\sum_{j=1}^m \mathbb{1}_{\left\{ x_j=k \right\}} + \sum_{i=1}^n q(y_i;\theta_k^t)} \right)
\end{align}
where $\beta(x) =\frac{\partial \alpha(x)}{ \partial x}$. 

\subsection{Gaussian mixture model}
The samples $(x_1,y_1),\dots,(x_m,y_m)$ are generated in an i.i.d fashion from the Gaussian mixture distribution
\begin{align} \label{eqn:Gaussfamilyjoint}
p_{XY}(k,y;\theta)&= p_X(k) p_{Y|X}(y|k;\theta)\\
&=\pi_k \frac{1}{\sqrt{2\pi}} e^{-\frac{(y-\theta_k)^2}{2}}.
\end{align}
The sample $y_1,\dots,y_n$ are generated in an i.i.d fashion from the induced distribution 
\begin{align} \label{eqn:Gaussfamilymarg}
p_Y(y;\theta)=\sum_{k=1}^K \pi_k \frac{1}{\sqrt{2 \pi}}e^{-\frac{(y-\theta_k)^2}{2}}.
\end{align}
By plugging (\ref{eqn:Gaussfamilymarg}) and (\ref{eqn:Gaussfamilyjoint}) into (\ref{eqn:Qnm}) we compute $Q_{n,m}(\theta;\theta^t)$ for the Gaussian mixture model as follows
\begin{align} \label{eqn:Qnm_comp_Gaussfam}
&Q_{n,m}(\theta;\theta^t) = \nonumber \\
&-\frac{1}{n+m} \left( \sum_{j=1}^m \sum_{k=1}^K \mathbb{1}_{\left\{ x_j=k \right\}} \left( \frac{(y_j-\theta_k)^2}{2} +\log(\frac{1}{\pi_k} \sqrt{2\pi}) \right) \right. \nonumber \\
& \quad \quad  \left. +\sum_{i=1}^n \sum_{k=1}^K q(y_i;\theta_k^t) \left( \frac{(y_i-\theta_k)^2}{2} +\log(\frac{1}{\pi_k} \sqrt{2\pi}) \right) \right)
\end{align}
where,
\begin{align}
q(y_i;\theta_k^t) =\frac{e^{-\frac{(y_i-\theta_k^t)^2}{2}}}{\sum_{k^{\prime}=1}^K e^{-\frac{(y_i-\theta_{k^{\prime}}^t)^2}{2}}}.
\end{align}
Plugging the computed $Q_{n,m}(\theta;\theta^t)$ of (\ref{eqn:Qnm_comp_Gaussfam}) into (\ref{eqn:Mnm}) we compute $M_{n,m}(\theta_k^t)$ as follows 
\begin{align}  \label{eqn:Mnm_comp_Gaussfam}
M_{n,m}(\theta_k^t)= \frac{\sum_{j=1}^m \mathbb{1}_{\left\{ x_j=k \right\}}y_j + \sum_{i=1}^n q(y_i;\theta_k^t) y_i }{\sum_{j=1}^m \mathbb{1}_{\left\{ x_j=k \right\}} + \sum_{i=1}^n q(y_i;\theta_k^t)}.
\end{align}
\subsection{Symmetric mixture of two Gaussians} \label{sec:SymtwoGauss}
The unlabeled samples are i.i.d. generated from (\ref{eqn:Gaussfamilymarg}) and labeled samples are i.i.d. generated from (\ref{eqn:Gaussfamilyjoint}) for $K=2$ and $\theta_1=-\theta$ and $\theta_2=\theta$. We compute $Q_{n,m}(\theta;\theta^t)$ as in (\ref{eqn:Qnm_comp_Gaussfam}) for $K=2$, $\theta_1=-\theta$ and $\theta_2=\theta$, however $M_{n,m}(\theta_k^t)$ has a slightly different form compared to (\ref{eqn:Mnm_comp_Gaussfam}) because the means of the Gaussian mixture models with two components depend on each other, thus
\begin{align} \label{eqn:update_Mstep_two_Gaussian}
M_{n,m}(\theta^t)=\frac{\sum_{j=1}^m (1-2 \mathbb{1}_{\left\{ x_j=1 \right\}}) y_j + \sum_{i=1}^n (1-2q(y_i;\theta^t)) y_i}{m+n}
\end{align}
where
\begin{align} \label{eqn:q_fun_GMM2sym}
q(y_i;\theta^t) =\frac{e^{-\frac{(y_i+\theta^t)^2}{2}}}{e^{-\frac{(y_i+\theta^t)^2}{2}}+ e^{-\frac{(y_i-\theta^t)^2}{2}}}.
\end{align}
\section{Main Results}
Let us assume that the true parameter of the model is $\theta_*$. The global convergence of EM to the true parameter $\theta_*$ occurs if, 
\begin{align}
 |M_{n,m}(\theta_k) - \theta_{k_*}| \leq r | \theta_k-\theta_{k_*}|,
\end{align}
where $r <1$, is a contraction coefficient that determines the convergence rate. We will split the analysis as follows
\begin{align}
 |M_{n,m}(\theta_k) - \theta_{k_*}| &\leq \beta |M_{n}(\theta_k)-\theta_{k_*}| \label{eqn:impactlabel} \\
 & \leq \beta \kappa | \theta_k-\theta_{k_*}|, \label{eqn:unimpactlabel}
\end{align}
where $\beta <1$, is a contraction coefficient that accounts for the impact of the labeled sampled and $\kappa$ accounts for the impact of unlabelled samples.
The overall convergence rate $r$ bears the multiplicative property i.e. $r=\kappa \beta$. We primarily focus on $\beta$, that is the impact of the labeled samples on the convergence rate. We will assume that $\gamma=\frac{m}{m+n}$ where $ 0 \leq \gamma \leq 1$ is a constant. For the population based EM, as $n,m \to \infty$ implies that $M_{n,m}(\theta) \to M_{\gamma}(\theta)$, $M_{n}(\theta) \to M_0(\theta)$ and
\begin{align} \label{eqn:limit_gamma_sample_improve}
 |M_{\gamma}(\theta_k) - \theta_{k_*}| \leq \beta_k |M_0(\theta_k)-\theta_{k_*}|.
\end{align}
Let us assume that we can choose the number of labeled samples $m_1,m_2,\dots,m_K$ from each cluster such that $\sum_{k=1}^K m_k=m$. For the population EM, as $m\to \infty$ implies that $\frac{m_k}{m} \to \pi_k$, thus instead of the samples for each cluster we will choose $\pi_1,\pi_2,\dots,\pi_K$.

\begin{theorem} \label{thm:Gaussmix}
For the population EM, consider the Gaussian mixture model in (\ref{eqn:Gaussfamilymarg}). The labelled samples contribute to the convergence rate with the contraction coefficient $\beta_k=\frac{c_{\theta_k}}{\frac{ \pi_k \gamma}{1-\gamma}+c_{\theta_k}}$ that satisfies 
\begin{align} \label{eqn:contractiongain_of_gamma} 
 |M_{\gamma}(\theta_k) - \theta_{k_*}| \leq \beta_k |M_0(\theta_k)-\theta_{k_*}|,
\end{align}
for $k=1,\dots,K$ where $c_{\theta_k}=\mathbb{E}[q(Y;\theta_k)]$.
\end{theorem}


\begin{theorem} \label{thm:Expfammix}
For the population EM, consider the Exponential family mixture model in (\ref{eqn:expfamilymarg}). Assume that the parameters of EM are initialized within the neighborhood of the true parameters $\theta \in (\theta_*-\epsilon,\theta_*+\epsilon)$, $\epsilon>0$. Then, the labelled samples contribute to the convergence rate with the contraction coefficient $\beta_k=\frac{c_{\theta_k}}{\frac{\pi_k \gamma}{1-\gamma}+c_{\theta_k}}$ that satisfies 
\begin{align} \label{eqn:approx_contractiongain_of_gamma} 
 |M_{\gamma}(\theta_k) - \theta_{k_*}| \approx \beta_k |M_0(\theta_k)-\theta_{k_*}|,
\end{align}
for $k=1,\dots,K$ where $c_{\theta_k}=\mathbb{E}[q(Y;\theta_k)]$.
\end{theorem}

\begin{theorem} \label{thm:Gausstwo}
For the population EM, consider the Gaussian mixture model with two symmetric components in Section \ref{sec:SymtwoGauss}. Then, 
\begin{enumerate}
\item for $\theta_*>\frac{2}{e} \sqrt{1-\gamma}$, the population EM convergences globally for all the initial conditions $\theta>\theta_*$ with a contraction coefficient $r(\theta_*)=(1-\gamma)\frac{4}{\theta_*^2 e^2}$;
\item for $\theta_*>2$, the population EM convergences globally for all the initial conditions $\theta>\theta_*$ with a contraction coefficient $r(\theta_*)=(1-\gamma)e^{-c \theta_*^2}$;
\item for $\theta_*>\frac{1}{2}$, the population EM convergences globally for all the initial conditions $\theta>\theta_*+1$ with a contraction coefficient $R(\theta_*)=(1-\gamma)e^{-c \theta_*^2}$;
\end{enumerate}
satisfying  $|M_{\gamma}(\theta) - \theta_*| \leq r |\theta-\theta_*|$ for some constant $c$.
\end{theorem}
\section{Conclusion and Discussion}
In this paper, we investigate the Expectation Maximization algorithm in the context of semi-supervised learning. The analysis is split into two main steps that are: 1) the impact of labeled samples given by contraction coefficient $\beta$ in (\ref{eqn:impactlabel}) and 2) the impact of the unlabeled samples given by $\kappa$ in (\ref{eqn:unimpactlabel}). The overall convergence rate $r$ bears the multiplicative property $r=\beta \kappa$. We primarily focus on the impact of the labeled samples on the convergence of EM. Labeled samples: 1) accelerate the convergence of EM, 2) help the EM converge when it is not convergent. Here is an example of the second case. 
\begin{example}
Let $\kappa$ be the non-contraction coefficient, $\kappa>1$, meaning that EM with unlabeled samples does not converge. The labeled samples introduce a contraction coefficient $\beta<\frac{1}{\kappa}<1$, thus now EM converges with convergence rate $r=\beta\kappa <1$.
\end{example}

\section{Proof of Theorem \ref{thm:Gaussmix}}
For the population EM, $M_{n,m}(\theta_k) \to M_{\gamma}(\theta_k)$ and $M_n(\theta_k) \to M_0(\theta_k)$, then (\ref{eqn:Mnm_comp_Gaussfam}) becomes

\begin{align}
M_{\gamma}(\theta_k)=\frac{(1-\gamma) \mathbb{E}[ q(Y;\theta_k) Y]+ \gamma \mathbb{E}[\mathbb{1}_{\left\{X=k\right\}}Y]}{(1-\gamma) \mathbb{E}[ q(Y;\theta_k)]+ \gamma \mathbb{E}[\mathbb{1}_{\left\{X=k\right\}}]}
\end{align}
and 
\begin{align} \label{eqn:infMGauss}
M_{0}(\theta_k)=\frac{ \mathbb{E}[ q(Y;\theta_k) Y]}{\mathbb{E}[ q(Y;\theta_k)]}.
\end{align}
Let us compute $\theta_{{k_*}}$ directly by setting $\gamma \to 1$, meaning that all the samples are labeled and already clustered, thus 
\begin{align} \label{eqn:theta_k_star}
\theta_{{k_*}}=\frac{\mathbb{E}[\mathbb{1}_{\left\{X=k\right\}} Y]}{ \mathbb{E}[ \mathbb{1}_{\left\{X=k\right\}}]}.
\end{align}
We know that $\theta_{{k_*}}=M_0(\theta_{{k_*}})$ is a fixed point of the function, 
then from (\ref{eqn:infMGauss})  
\begin{align}  \label{eqn:theta_k_startwo}
\theta_{{k_*}}=\frac{\mathbb{E}[q(Y;\theta_{{k_*}})Y]}{\mathbb{E}[q(Y;\theta_{{k_*}})]}.
\end{align}
Now let us compute $|M_0(\theta_k) -\theta_{{k_*}}|$ as follows
\begin{align}
|M_0(\theta_k) -\theta_{{k_*}}| &=|M_0(\theta_k) -M_0(\theta_{{k_*}})| \\
&= \left| \frac{\mathbb{E}[q(Y;\theta_{k})Y]}{\mathbb{E}[q(Y;\theta_{k})]}- \frac{\mathbb{E}[q(Y;\theta_{{k_*}})Y]}{\mathbb{E}[q(Y;\theta_{{k_*}})]} \right|. \label{eqn:Mzero_the_minusthe}
\end{align}
Similarly let us compute $|M_{\gamma}(\theta_k) -\theta_{{k_*}}|$ as follows
\begin{align}
|M_{\gamma}(\theta_k) -\theta_{{k_*}}| &=|M_{\gamma}(\theta_k) -M_{\gamma}(\theta_{{k_*}})| \\
&= \left| \frac{(1-\gamma)\mathbb{E}[q(Y;\theta_k)Y]+ \gamma \mathbb{E}[ \mathbb{1}_{\left\{X=k\right\}}Y]}{(1-\gamma)\mathbb{E}[q(Y;\theta_k)]+\gamma \mathbb{E}[ \mathbb{1}_{\left\{X=k\right\}}]} \right. \nonumber \\
&\quad \left. - \frac{(1-\gamma) \mathbb{E}[q(Y;\theta_{{k_*}})Y]+ \gamma \mathbb{E}[ \mathbb{1}_{\left\{X=k\right\}}Y]}{(1-\gamma)\mathbb{E}[q(Y;\theta_{{k_*}})]+\gamma \mathbb{E}[ \mathbb{1}_{\left\{X=k\right\}}]} \right| \label{eqn:Mgamma_the_minusthe}
\end{align}
From the assumption, we can choose the number of labeled samples of each cluster, and in the limiting case for the population EM we can choose $\pi_1,\dots,\pi_K$. We have that
\begin{align} \label{eqn:GMMuniformprior}
\mathbb{E}[ \mathbb{1}_{\left\{X=k\right\}}]=\pi_k.
\end{align} 
Let us define 
\begin{align} \label{eqn:defck}
c_{\theta_k}=\mathbb{E}[q(Y;\theta_k)].
\end{align} 
Without loss of optimality we can assume that
\begin{align}
\frac{\mathbb{E}[q(Y;\theta_k)Y]}{\mathbb{E}[q(Y;\theta_k)]} \geq \frac{\mathbb{E}[ \mathbb{1}_{\left\{X=k\right\}}Y]}{\mathbb{E}[ \mathbb{1}_{\left\{X=k\right\}}]}
\end{align}
thus, we can write the above equation in the following form  
\begin{align} \label{eqn:express_eta}
\frac{\mathbb{E}[q(Y;\theta_k)Y]}{\mathbb{E}[q(Y;\theta_k)]} =\eta \frac{\mathbb{E}[ \mathbb{1}_{\left\{X=k\right\}}Y]}{\mathbb{E}[ \mathbb{1}_{\left\{X=k\right\}}]}
\end{align}
where $\eta \geq 1$. By plugging (\ref{eqn:theta_k_star}), (\ref{eqn:theta_k_startwo}) and (\ref{eqn:express_eta}) into (\ref{eqn:Mzero_the_minusthe}) we compute $|M_0(\theta_k) -\theta_{k_*}|$ as follows
\begin{align} \label{eqn:compMzero_terms_eta}
|M_0(\theta_k) -\theta_{k_*}| =(\eta -1) |\theta_{k_*}|.
\end{align}
By plugging (\ref{eqn:theta_k_star}), (\ref{eqn:theta_k_startwo}), (\ref{eqn:GMMuniformprior}), (\ref{eqn:defck}) and (\ref{eqn:express_eta}) into (\ref{eqn:Mgamma_the_minusthe}) we compute $|M_{\gamma}(\theta_k) -\theta_{k_*}|$ as follows
\begin{align} \label{eqn:compMgamma_terms_eta}
|M_{\gamma}(\theta_k) -\theta_{k_*}| = \left( \frac{(1-\gamma) \eta c_{\theta_k} +\gamma \pi_k } {(1-\gamma) c_{\theta_k} +\gamma \pi_k } -1 \right)|\theta_{k_*}|.
\end{align}
By plugging (\ref{eqn:compMzero_terms_eta}) and (\ref{eqn:compMgamma_terms_eta}) into (\ref{eqn:limit_gamma_sample_improve}) we obtain the following
\begin{align}
\frac{(1-\gamma) \eta c_{\theta_k} +\gamma \pi_k}{(1-\gamma) c_{\theta_k} +\gamma \pi_k} -1 \leq \beta_k(\eta-1)
\end{align}
that simplifies as follows 
\begin{align}
\beta_k \geq \frac{c_{\theta_k}}{\frac{\pi_k \gamma}{1-\gamma}+c_{\theta_k}}.
\end{align}
To conclude, the labeled samples contribute to the convergence of the population EM with the contractive coefficient $\beta_k<1$.

\section{Proof of Theorem \ref{thm:Expfammix}}
For the population EM, $M_{n,m}(\theta_k) \to M_{\gamma}(\theta_k)$ and $M_n(\theta_k) \to M_0(\theta_k)$, then (\ref{eqn:Exp_fam_update}) becomes 

\begin{align} \label{eqn:alpha_prim_update}
\alpha^{\prime} (M_{\gamma}(\theta_k))=\frac{(1-\gamma) \mathbb{E}[ q(Y;\theta_k) Y]+ \gamma \mathbb{E}[\mathbb{1}_{\left\{X=k\right\}}Y]}{(1-\gamma) \mathbb{E}[ q(Y;\theta_k)]+ \gamma \mathbb{E}[\mathbb{1}_{\left\{X=k\right\}}]}.
\end{align}
where $\alpha^{\prime}(x):=\frac{\partial \alpha(x)}{\partial x}$. Setting $\gamma \to 1$, we obtain $M_1(\theta_k)=\theta_{{k_*}}$ and
\begin{align} \label{eqn:theta_k_star_exp_fam}
\alpha^{\prime}(\theta_{{k_*}})=\frac{\mathbb{E}[\mathbb{1}_{\left\{X=k\right\}} Y]}{ \mathbb{E}[ \mathbb{1}_{\left\{X=k\right\}}]}.
\end{align}
Note that $\theta_{{k_*}}=M_0(\theta_{{k_*}})$ is a fixed point of the function, then from (\ref{eqn:alpha_prim_update})  
\begin{align}  \label{eqn:theta_k_star_exp_fam_two}
\alpha^{\prime}(\theta_{{k_*}})=\frac{\mathbb{E}[q(Y;\theta_{{k_*}})Y]}{\mathbb{E}[q(Y;\theta_{{k_*}})]}.
\end{align}
From the assumption, we can choose the number of labeled samples of each cluster, and in the limiting case for the population EM we can choose $\pi_1,\dots,\pi_K$. We have that
\begin{align} \label{eqn:expfamuniformprior}
\mathbb{E}[ \mathbb{1}_{\left\{X=k\right\}}]=\pi_k.
\end{align} 
Let us define
\begin{align} \label{eqn:defck_exp}
c_{\theta_k}=\mathbb{E}[q(Y;\theta_k)].
\end{align} 
Without loss of optimality we can assume that $\eta \geq 1$ in the following equation
\begin{align} \label{eqn:express_eta_exp_fam}
\frac{\mathbb{E}[q(Y;\theta_k)Y]}{\mathbb{E}[q(Y;\theta_k)]} =\eta \frac{\mathbb{E}[ \mathbb{1}_{\left\{X=k\right\}}Y]}{\mathbb{E}[ \mathbb{1}_{\left\{X=k\right\}}]}.
\end{align}
By first order Taylor expansion we obtain
\begin{align} \label{eqn:Taylor_expfam_gamm}
\alpha^{\prime}(M_{\gamma}(\theta_k)) \hspace{-0.25em} \approx \hspace{-0.25em} \alpha^{\prime}(M_{\gamma}(\theta_{k_*})) \hspace{-0.25em} + \hspace{-0.25em} (M_{\gamma}(\theta_k) \hspace{-0.25em} - \hspace{-0.25em} M_{\gamma}(\theta_{k_*})) \mathcal{I}(M_{\gamma}(\theta_{k_*}))
\end{align}
where $\mathcal{I}(.)$ is the Fisher information. Similarly, for $\gamma=0$ the first order Taylor expansion becomes
\begin{align} \label{eqn:Taylor_expfam_zero}
\alpha^{\prime}(M_0(\theta_k)) \hspace{-0.25em} \approx \hspace{-0.25em} \alpha^{\prime}(M_0(\theta_{k_*})) \hspace{-0.25em} + \hspace{-0.25em} (M_0(\theta_k) \hspace{-0.25em} - \hspace{-0.25em} M_0(\theta_{k_*})) \mathcal{I}(M_0(\theta_{k_*})).
\end{align}
By combining (\ref{eqn:Taylor_expfam_gamm}) and (\ref{eqn:Taylor_expfam_zero}) we obtain
\begin{align} \label{eqn:approx_alpha_expfam}
\frac{|M_{\gamma}(\theta_k)-\theta_{k_*}|}{|M_0(\theta_k)-\theta_{k_*}|} \approx \frac{|\alpha^{\prime}(M_{\gamma}(\theta_k)) - \alpha^{\prime}(\theta_{k_*})|}{|\alpha^{\prime}(M_0(\theta_k)) -\alpha^{\prime}(\theta_{k_*})|}
\end{align}
We compute $|\alpha^{\prime}(M_0(\theta_k)) -\alpha^{\prime}(\theta_{k_*})|$ as follows
\begin{align}
|\alpha^{\prime}(M_0(\theta_k)) -\alpha^{\prime}(\theta_{k_*})| &= \left| \frac{\mathbb{E}[q(Y;\theta_{k})Y]}{\mathbb{E}[q(Y;\theta_{k})]}- \frac{\mathbb{E}[q(Y;\theta_{{k_*}})Y]}{\mathbb{E}[q(Y;\theta_{{k_*}})]} \right| \\
&=(\eta-1) |\alpha^{\prime}(\theta_{k_*})| \label{eqn:comp_diff_alpha_zero}
\end{align}
where the last equality follows from (\ref{eqn:theta_k_star_exp_fam}), (\ref{eqn:theta_k_star_exp_fam_two}) and (\ref{eqn:express_eta_exp_fam}). Similarly, we compute $|\alpha^{\prime}(M_{\gamma}(\theta_k)) -\alpha^{\prime}(\theta_{k_*})|$ as follows
\begin{align}
|\alpha^{\prime}(M_{\gamma}(\theta_k)) &-\alpha^{\prime}(\theta_{k_*})| \nonumber \\
&= \left| \frac{(1-\gamma)\mathbb{E}[q(Y;\theta_k)Y]+ \gamma \mathbb{E}[ \mathbb{1}_{\left\{X=k\right\}}Y]}{(1-\gamma)\mathbb{E}[q(Y;\theta_k)]+\gamma \mathbb{E}[ \mathbb{1}_{\left\{X=k\right\}}]} \right. \nonumber \\
& \left. - \frac{(1-\gamma) \mathbb{E}[q(Y;\theta_{{k_*}})Y]+ \gamma \mathbb{E}[ \mathbb{1}_{\left\{X=k\right\}}Y]}{(1-\gamma)\mathbb{E}[q(Y;\theta_{{k_*}})]+\gamma \mathbb{E}[ \mathbb{1}_{\left\{X=k\right\}}]} \right| \\
&= \left( \frac{(1-\gamma) \eta c_{\theta_k} +\gamma \pi_k}{(1-\gamma) c_{\theta_k} +\gamma \pi_k } -1 \right)|\alpha^{\prime}(\theta_{k_*})| \label{eqn:comp_diff_alpha_gamma}
\end{align}
where the last equality follows from (\ref{eqn:theta_k_star_exp_fam})-(\ref{eqn:express_eta_exp_fam}). By plugging (\ref{eqn:comp_diff_alpha_zero}) and (\ref{eqn:comp_diff_alpha_gamma}) into (\ref{eqn:approx_alpha_expfam}) the expression simplifies as follows
\begin{align}
\frac{|M_{\gamma}(\theta_k)-\theta_{k_*}|}{|M_0(\theta_k)-\theta_{k_*}|} \approx \frac{c_{\theta_k}}{c_{\theta_k}+ \frac{\pi_k \gamma}{1-\gamma}}.
\end{align}

\section{Proof of Theorem \ref{thm:Gausstwo}}
Initially let us bound $|M_{\gamma}(\theta) -\theta_*| \leq \beta_k |M_0(\theta) -\theta_*|$, $k=1,2$, for the Gaussian mixture model with two symmetric components given in Section \ref{sec:SymtwoGauss}. Then, from Theorem \ref{thm:Gaussmix}, we obtain $\beta_k=1-\gamma$ by computing $c_{\theta}=c_{-\theta}=\frac{1}{2}$ as follows
\begin{align}
c_{\theta}&=\int_{0}^{\infty} q(y;\theta) p_Y(y) dy+ \int_{0}^{\infty} q(-y;\theta) p_Y(-y) dy \\
&=\int_{0}^{\infty} q(y;\theta) p_Y(y) dy+ \int_{0}^{\infty} \left(1-q(y;\theta) \right) p_Y(y) dy \\
&=\frac{1}{2}.
\end{align}
So far we have bounded 
\begin{align}
|M_{\gamma}(\theta) -\theta_*| \leq (1-\gamma) |M_0(\theta) -\theta_*|.
\end{align}
Last step is to bound $|M_0(\theta) -\theta_*| \leq \kappa |\theta- \theta_*|$.
\subsection{Item 1)}
For Item 1), we proceed by stating the following lemma, 
\begin{lemma} \label{lemma:AltEM}
Let $M_0(\theta)$ be concave and increasing, thus for $\theta>\theta_*$,
\begin{align}
|M_0(\theta)-\theta_*| \leq \frac{\partial M_0(\theta_*)}{\partial \theta_*} |\theta-\theta_*|.
\end{align}
\end{lemma}

\begin{proof}
By concavity we have 
\begin{align}
M_0(\theta) -M_0(\theta_*) \leq \frac{\partial M_0(\theta_*)}{\partial \theta_*} (\theta-\theta_*), \label{eqn:TaylorMVT}
\end{align}
where $M_0(\theta_*)=\theta_*$, is a fixed point solution. By using the fact that $M_0$ is increasing than for $\theta>\theta_*>0$, $M_0(\theta) > M_0(\theta_*)$, thus we can take absolute values on both sides of (\ref{eqn:TaylorMVT}) and the sign do not change, which completes the proof.
\end{proof}

Apply Lemma \ref{lemma:AltEM}, on $M_0(\theta)$ that is computed in (\ref{eqn:update_Mstep_two_Gaussian}), $M_0(\theta_*)=-2\mathbb{E}[q(Y;\theta_*)Y]$, thus 
\begin{align}
\frac{\partial M_0(\theta_*)}{\partial \theta_*} &=-2 \mathbb{E} \left[ Y \frac{\partial q(Y;\theta_*)}{\partial \theta_*} \right] \\
&=4\mathbb{E}\left[ \frac{Y^2}{(e^{-Y\theta_*}+e^{Y\theta_*})^2} \right] \geq 0. \label{eqn:fullcompfirstder}
\end{align}
where (\ref{eqn:fullcompfirstder}) follows from
\begin{align} \label{eqn:der_q_GMM2}
\frac{\partial q(Y;\theta_*)}{\partial \theta_*}&=\frac{\partial }{\partial \theta_*} \frac{e^{-\frac{(Y+\theta_*)^2}{2}}}{e^{-\frac{(Y-\theta_*)^2}{2}}+e^{-\frac{(Y+\theta_*)^2}{2}}} \\
&=\frac{\partial }{\partial \theta_*} \frac{1}{1+e^{2Y \theta_*}} \\
&=-\frac{2Y}{\left( e^{-Y \theta_*}+e^{Y \theta_*} \right)^2}.
\end{align}
On the other hand we will show the concavity of $M_0$ as follows
\begin{align}
\frac{\partial^2 M_0(\theta_*)}{\partial \theta_*^2} &=-2 \mathbb{E} \left[ Y \frac{\partial^2 q(Y;\theta_*)}{\partial \theta_*^2} \right] \\
&=-8\mathbb{E}\left[ \frac{Y^3(e^{Y\theta_*}-e^{-Y\theta_*})}{(e^{-Y\theta_*}+e^{Y\theta_*})^3} \right] \leq 0. \label{eqn:fullcompsecondder}
\end{align}
where for $\theta_*>0$, (\ref{eqn:fullcompsecondder}) follows from $Y(e^{Y\theta_*}-e^{-Y\theta_*}) \geq 0$. The contraction coefficient $\kappa$ is bounded as follows
\begin{align}
\frac{\partial M_0(\theta_*)}{\partial \theta_*} &=4\mathbb{E}\left[ \frac{Y^2}{(e^{-Y\theta_*}+e^{Y\theta_*})^2} \right] \\
& \leq 4 \mathbb{E} \left[ Y^2 e^{-2|Y|\theta_*} \right] \label{eqn:sumexpineqexp} \\
&\leq 4 \sup_{0\leq t} t^2 e^{-2t\theta_*} \label{eqn:supintone} \\
&=\frac{4}{\theta_*^2 e^2}. \label{eqn:supintonecomp}
\end{align}
where (\ref{eqn:sumexpineqexp}) follows from $e^{-Y\theta_*}+e^{Y\theta_*} \geq e^{|Y|\theta_*}$, (\ref{eqn:supintone}) is bounded by taking the supremum inside the expectation and the supremum of (\ref{eqn:supintonecomp}) happens at $t=\frac{1}{\theta_*}$. Thus, $\frac{\partial M_0(\theta_*)}{\partial \theta_*}$ which is a contraction coefficient, must be $\frac{\partial M_0(\theta_*)}{\partial \theta_*} \leq 1$, which happens for $\theta_* \geq \frac{2}{e}$. 
\subsection{Item 2)}
To obtain faster convergence rates we define the set $\mathcal{E}=\left\{ \tilde{Y} \leq \frac{\theta_*}{4} \right\}$, where $\tilde{Y} \sim \mathcal{N}(\theta_*,1)$, thus

\begin{align}
\frac{1}{4} &\frac{\partial M_0(\theta_*)}{\partial \theta_*}  \leq \mathbb{E} \left[ Y^2 e^{-2|Y|\theta_*} \right] \label{eqn:ineqderbound}  \\
&= \mathbb{E} \left[ \tilde{Y}^2 e^{-2|\tilde{Y}|\theta_*} \right] \\
&=P(\mathcal{E}) \mathbb{E} \left[ \tilde{Y}^2 e^{-2|\tilde{Y}|\theta_*} |\mathcal{E} \right] + P(\mathcal{E}^c) \mathbb{E} \left[ \tilde{Y}^2 e^{-2|\tilde{Y}|\theta_*} |\mathcal{E}^c \right] \\
&\leq P(\mathcal{E}) \mathbb{E} \left[ \tilde{Y}^2 e^{-2|\tilde{Y}|\theta_*} |\mathcal{E} \right] + \mathbb{E} \left[ \tilde{Y}^2 e^{-2|\tilde{Y}|\theta_*} |\mathcal{E}^c \right] \label{eqn:iterexpecff}
\end{align}
where (\ref{eqn:ineqderbound}) follows from (\ref{eqn:sumexpineqexp}). Consider each term separately in the last expression, $P(\mathcal{E})=P\left(\tilde{Y}-\theta_*\leq -\frac{3\theta_*}{4} \right) \leq e^{-\frac{9\theta_*^2}{32}}$ that follows from Gaussian tail bound.  
We bound the second term in (\ref{eqn:iterexpecff}) as follows
\begin{align}
\mathbb{E} \left[ \tilde{Y}^2 e^{-2|\tilde{Y}|\theta_*} |\mathcal{E} \right] &= \int_{-\infty}^{\frac{\theta_*}{4}} \tilde{y}^2 e^{-2|\tilde{y}|\theta_*}  \frac{1}{\sqrt{2 \pi}} e^{-\frac{(\tilde{y}-\theta_*)^2}{2}}d\tilde{y} \\ 
&\leq \int_{-\infty}^{\infty} \tilde{y}^2 e^{-2|\tilde{y}|\theta_*}  \frac{1}{\sqrt{2 \pi}} e^{-\frac{(\tilde{y}-\theta_*)^2}{2}}d\tilde{y} \\
&\leq \sup_{0\leq t} t^2 e^{-2t\theta_*} \\
&=\frac{1}{\theta_*^2 e^2}.
\end{align}
Last term in (\ref{eqn:iterexpecff}) is bounded as follows
\begin{align}
\mathbb{E} \left[ \tilde{Y}^2 e^{-2|\tilde{Y}|\theta_*} |\mathcal{E}^c \right] &= \int_{\frac{\theta_*}{4}}^{\infty} \tilde{y}^2 e^{-2|\tilde{y}|\theta_*}  \frac{1}{\sqrt{2 \pi}} e^{-\frac{(\tilde{y}-\theta_*)^2}{2}}d\tilde{y} \\ 
&\leq \sup_{t\geq \frac{\theta_*}{4}} t^2 e^{-2t\theta_*} \\
&=\frac{\theta_*^2}{16} e^{-\frac{\theta_*^2}{2}}
\end{align}
where supremum happens at $t=\frac{\theta_*}{4}$ for $\frac{\theta_*}{4} \geq \frac{1}{\theta_*}$, that is for $\theta_*\geq 2$. Combining all the equation together we obtain the following

\begin{align}
\frac{\partial M_0(\theta_*)}{\partial \theta_*} &= 4 \left( \frac{1}{\theta_*^2 e^2} e^{-\frac{9\theta_*^2}{32}} + \frac{\theta_*^2}{16} e^{-\frac{\theta_*^2}{2}} \right)
\end{align}
where for $\theta_*>2$, we have a contraction coefficient $\frac{\partial M(\theta_*)}{\partial \theta_*} \leq 1$, thus the convergence rate is $e^{-c \theta_*^2}$ for some constant $c$.

\subsection{Item 3)}
\begin{lemma}[\cite{Wainwright17}] \label{thm:Wainright}
Let $Q(.;\theta_*)$ be any $\lambda-$strong concave function. Let $M(\theta)$ be defined as in (\ref{eqn:Mnm}), and $Q(.;\theta)$ be $\gamma-$gradient smooth that is $|\nabla Q(M(\theta);\theta_*) -\nabla Q(M(\theta);\theta)| \leq \gamma |\theta -\theta_*|$, thus
\begin{align}
|M(\theta) -\theta_*| \leq \frac{\gamma}{\lambda}|\theta -\theta_*|
\end{align}
for $0 \leq \gamma \leq \lambda$.
\end{lemma}

Apply lemma \ref{thm:Wainright}, on $Q_{n,m}$ defined in (\ref{eqn:Qnm_comp_Gaussfam}) which for $m=0$ and $n\to \infty$ will be $Q_0$. Note that $Q_0(.;\theta_*)$ is $1-$strong concave. We need to show that $Q_0(.;\theta)$ is $e^{-c \theta_*^2}-$gradient smooth. Starting from the definition of gradient smoothness we have
\begin{align}
|\nabla Q_0(M(\theta);\theta_*) &-\nabla Q_0(M(\theta);\theta)| \nonumber \\
&= 2\left| \mathbb{E}[ \left(q(Y;\theta)-q(Y;\theta_* \right)Y] \right| \\
&= 2 \left| f(\theta) -f(\theta_*) \right| \label{eqn:gradientsmoothdifffunction}
\end{align}
where we define $f(\theta):=-\mathbb{E}[q(Y;\theta)Y]$ and $q$ is defined in (\ref{eqn:q_fun_GMM2sym}). By the fundamental theorem of calculus 
\begin{align}
f(\theta) -f(\theta_*)= \int_{\theta_*}^{\theta} \frac{ d f(\theta_u)}{d \theta_u} d \theta_u,
\end{align}
where $\frac{d f(\theta_u)}{d \theta_u}$ is computed as follows
\begin{align}
\frac{ d f(\theta_u)}{d \theta_u}&=\mathbb{E} \left[Y \frac{d q(Y;\theta_u)}{d \theta_u} \right] \\
&=  \mathbb{E} \left[ \frac{2Y^2}{\left( e^{-Y \theta_u}+e^{Y \theta_u} \right)^2} \right]
\end{align}
where the last equation is computed in (\ref{eqn:der_q_GMM2}). The difference $f(\theta) -f(\theta_*)$ is bounded as follows
\begin{align}
f(\theta) -f(\theta_*)&= \int_{\theta_*}^{\theta} \mathbb{E} \left[ \frac{2Y^2}{\left( e^{-Y \theta_u}+e^{Y \theta_u} \right)^2} \right] d \theta_u \\
&\leq \int_{\theta_*}^{\theta} \mathbb{E} \left[ 2Y^2 e^{-2|Y| \theta_u} \right] d \theta_u \\
&=\mathbb{E} \left[ 2Y^2  \int_{\theta_*}^{\theta} e^{-2|Y| \theta_u} d \theta_u \right] \\
&=\mathbb{E} \left[ |Y| \left(  e^{-2|Y| \theta_*} -e^{-2|Y| \theta} \right) \right] \label{eqn:diffgradientsmooth}
\end{align}
where the upper bound is derived from the trivial inequality $e^{-Y \theta_u}+e^{Y \theta_u} \geq e^{|Y| \theta_u}$. Let us now evaluate $\mathbb{E} \left[ |Y| e^{-2|Y| x} \right]$ as follows
\begin{align}
\mathbb{E} \left[ |Y| e^{-2|Y| x} \right]
&= \underbrace{\int_{0}^{\infty} |y| e^{-2|y|x} \frac{1}{\sqrt{2\pi}} e^{-\frac{(y+\theta_*)^2}{2}} dy}_{A} \nonumber \\ &\quad \quad \quad \quad + \underbrace{\int_{0}^{\infty} |y| e^{-2|y|x}\frac{1}{\sqrt{2\pi}} e^{-\frac{(y-\theta_*)^2}{2}} dy}_{B}.
\end{align}
Let us compute both terms $(A)$ and $(B)$ in the last expression separately as follows
\begin{align}
A &= \int_{0}^{\infty} \frac{|y|}{\sqrt{2\pi}}e^{-\frac{(y+\theta_* +2x)^2}{2}} e^{2x^2 +2x \theta_*}dy \\
&=e^{2x^2 +2x \theta_*} \left( \int_{0}^{\infty} \frac{|y| +\theta_*+2x}{\sqrt{2\pi}}e^{-\frac{(y+\theta_* +2x)^2}{2}} dy \right. \nonumber \\
& \quad \quad \quad \quad \quad \quad \quad \quad \left. - \int_{0}^{\infty} \frac{\theta_*+2x}{\sqrt{2\pi}}e^{-\frac{(y+\theta_* +2x)^2}{2}} dy \right) \\
&= e^{2x^2 +2x \theta_*} \left( \int_{\frac{(\theta_*+2x)^2}{2}}^{\infty} \frac{1}{\sqrt{2\pi}} e^{-v}dv - \int_{\theta_* +2x}^{\infty} \frac{\theta_*+2x}{\sqrt{2\pi}}e^{-\frac{w^2}{2}} dw \right) \\
&=e^{2x^2 +2x \theta_*} \left( \frac{1}{\sqrt{2\pi}} e^{-\frac{(\theta_*+2x)^2}{2}}  -(\theta_*+2x)\Phi(\theta_*+2x) \right).
\end{align}
On the other side 
\begin{align}
B &= \int_{0}^{\infty} \frac{|y|}{\sqrt{2\pi}}e^{-\frac{(y-\theta_* +2x)^2}{2}} e^{2x^2 -2x \theta_*}dy \\
&=e^{2x^2 -2x \theta_*} \left( \int_{0}^{\infty} \frac{|y| -\theta_*+2x}{\sqrt{2\pi}}e^{-\frac{(y-\theta_* +2x)^2}{2}} dy \right. \nonumber \\
& \quad \quad \quad \quad \quad \quad \quad \quad \left. - \int_{0}^{\infty} \frac{2x-\theta_*}{\sqrt{2\pi}}e^{-\frac{(y-\theta_* +2x)^2}{2}} dy \right) \\
&= e^{2x^2 -2x \theta_*} \left( \int_{\frac{(2x-\theta_*)^2}{2}}^{\infty} \frac{1}{\sqrt{2\pi}} e^{-v}dv - \int_{2x-\theta_*}^{\infty} \frac{2x-\theta_*}{\sqrt{2\pi}}e^{-\frac{w^2}{2}} dw \right) \\
&=e^{2x^2 -2x \theta_*} \left( \frac{1}{\sqrt{2\pi}} e^{-\frac{(2x-\theta_*)^2}{2}}  -(2x-\theta_*)\Phi(2x-\theta_*) \right).
\end{align}

By plugging $(A)$ and $(B)$ into (\ref{eqn:diffgradientsmooth}) we obtain un upper bound on $f(\theta)-f(\theta_*)$ as follows
\begin{align}
& f(\theta)-f(\theta_*) \nonumber  \\
&\leq \underbrace{e^{4\theta_*^2} \left[ \frac{1}{\sqrt{2\pi}}e^{-\frac{9}{2}\theta_*^2} -3\theta_* \Phi(3\theta_*) \right]}_{C} \nonumber \\
& \quad \underbrace{- e^{2 \theta^2 +2\theta \theta_*} \left[ \frac{1}{\sqrt{2\pi}}e^{-\frac{1}{2}(2\theta+\theta_*)^2} -(2\theta+\theta_*)\Phi(2\theta+\theta_*) \right]}_{D} \nonumber \\
& \quad \underbrace{- e^{2 \theta^2 -2\theta \theta_*} \left[ \frac{1}{\sqrt{2\pi}}e^{-\frac{1}{2}(2\theta-\theta_*)^2} -(2\theta-\theta_*)\Phi(2\theta-\theta_*) \right]}_{E}.
\end{align}
We will further upper bound the above expression by making use of the following lemma.
\begin{lemma}[Deconcentration of Gaussians] \label{lemma:tailofnormal}
The cumulative density function ($\mathbb{P}(Y>t)$ or $\Phi(t)$) of a normal distribution $Y\sim \mathcal{N}(0,1)$ is bounded as follows $\forall t>0$,
\begin{align}
\left( \frac{1}{t} -\frac{1}{t^3}\right) \frac{1}{\sqrt{2\pi}} e^{-\frac{t^2}{2}} \leq \Phi(t) \leq \frac{1}{t} \frac{1}{\sqrt{2\pi}} e^{-\frac{t^2}{2}}.
\end{align}
\end{lemma}
\begin{proof}
For the upper bound $t>0$, 
\begin{align}
\Phi(t)&=\int_t^{\infty} \frac{1}{\sqrt{2\pi}} e^{-\frac{y^2}{2}} dy \\
&=\int_0^{\infty} \frac{1}{\sqrt{2\pi}} e^{-\frac{(x+t)^2}{2}} dx \\
&= e^{-\frac{t^2}{2}}\int_0^{\infty} \frac{1}{\sqrt{2\pi}} \underbrace{e^{-\frac{x^2}{2}}}_{\leq 1} e^{-xt} dx \\
& \leq e^{-\frac{t^2}{2}}\int_0^{\infty} \frac{1}{\sqrt{2\pi}} e^{-xt} dx \\
&=\frac{1}{t \sqrt{2\pi}} e^{-\frac{t^2}{2}}.
\end{align}
For the lower bound,
\begin{align}
\Phi(t)&=\int_t^{\infty} \frac{1}{\sqrt{2\pi}} e^{-\frac{y^2}{2}} dy \\
&> \int_t^{\infty} \frac{1-3y^{-4}}{\sqrt{2\pi}} e^{-\frac{y^2}{2}} dy \\
&= \left(\frac{1}{t} -\frac{1}{t^3}\right) \frac{1}{\sqrt{2\pi}} e^{-\frac{t^2}{2}},
\end{align}
where the last part is an identity that is verified by the fundamental theorem of calculus.
\end{proof}
By using the upper bound in Lemma \ref{lemma:tailofnormal} we can show that $D<0$ and $E<0$ and by using the lower bound of Lemma \ref{lemma:tailofnormal} we obtain
\begin{align}
3\theta_*\Phi(3\theta_*) \geq \left( 1-\frac{1}{9\theta_*^2} \right) \frac{1}{\sqrt{2\pi}} e^{-\frac{9}{2} \theta_*^2}
\end{align}
thus 
\begin{align}
C= e^{4\theta_*^2} \left[ \frac{1}{\sqrt{2\pi}}e^{-\frac{9}{2}\theta_*^2} -3\theta_* \Phi(3\theta_*) \right] \leq \frac{1}{9\theta_*^2} \frac{1}{\sqrt{2\pi}} e^{-\frac{1}{2} \theta_*^2}.
\end{align}
By combining all the pieces together
\begin{align}
2|f(\theta) -f(\theta_*)| &\leq2( C+D+E) \leq 2C \\
&\leq  \frac{2}{9\theta_*^2} \frac{1}{\sqrt{2\pi}} e^{-\frac{1}{2} \theta_*^2} \\
&\leq  \frac{2}{9\theta_*^2} \frac{1}{\sqrt{2\pi}} e^{-\frac{1}{2} \theta_*^2} |\theta- \theta_*|
\end{align}
for $\theta > \theta_* +1$. For $\theta_*>\frac{1}{2}$ we have that $\frac{2}{(1+\alpha) 9\theta_*^2 \sqrt{2\pi}} e^{-\frac{1}{2} \theta_*^2} \leq 1$ and the convergence rate is $e^{-c \theta_*^2}$ for some constant c, thus
\begin{align}
|M_0(\theta) -\theta_*| \leq e^{-c \theta_*^2} |\theta- \theta_*|.
\end{align}

\section*{Acknowledgment}
This work was supported by the Swiss National Science Foundation, early postdoc mobility fellowship under Grant 199759.

\bibliographystyle{IEEEtran}
\bibliography{EM_semi-sup}

\begin{thebibliography}{10}
\providecommand{\url}[1]{#1}
\csname url@samestyle\endcsname
\providecommand{\newblock}{\relax}
\providecommand{\bibinfo}[2]{#2}
\providecommand{\BIBentrySTDinterwordspacing}{\spaceskip=0pt\relax}
\providecommand{\BIBentryALTinterwordstretchfactor}{4}
\providecommand{\BIBentryALTinterwordspacing}{\spaceskip=\fontdimen2\font plus
\BIBentryALTinterwordstretchfactor\fontdimen3\font minus
  \fontdimen4\font\relax}
\providecommand{\BIBforeignlanguage}[2]{{%
\expandafter\ifx\csname l@#1\endcsname\relax
\typeout{** WARNING: IEEEtran.bst: No hyphenation pattern has been}%
\typeout{** loaded for the language `#1'. Using the pattern for}%
\typeout{** the default language instead.}%
\else
\language=\csname l@#1\endcsname
\fi
#2}}
\providecommand{\BIBdecl}{\relax}
\BIBdecl

\bibitem{Wu83}
C.~Wu, ``On the convergence properties of the {EM} algorithm,'' \emph{The
  Annals of Statistics}, vol.~11, no.~1, pp. 95--103, 1983.

\bibitem{Hero-Fessler95}
A.~O. Hero and J.~A. Fessler, ``Convergence in norm for alternating
  expectation-maximization {(EM)} type algorithms,'' \emph{Statistica Sinica},
  vol.~5, pp. 41--54, 1995.

\bibitem{Meng94}
X.-L. Meng, ``On the rate of convergence of the {ECM} algorithm,'' \emph{The
  Annals of Statistics}, vol.~22, no.~1, pp. 326--339, 1994.

\bibitem{Meng-Rubin94}
X.-L. Meng and D.~B.Rubin, ``On the global and componentwise rates of
  convergence of the {EM} algorithm,'' \emph{Linear Algebra and its
  Applications}, vol. 199, no. 413-425, 1994.

\bibitem{Redner-Walker84}
R.~A. Redner and H.~F. Walker, ``Mixture densities, maximum likelihood and the
  {EM} algorithm,'' \emph{SIAM Review}, vol.~26, no.~2, 1984.

\bibitem{Dasgupta-Schulman07}
S.~Dasgupta and L.~Schulman, ``A probabilistic analysis of {EM} for mixtures of
  separated, spherical {G}aussians,'' \emph{Journal of Machine Learning
  Research}, vol.~8, pp. 203--226, 2007.

\bibitem{Xu-Jordan96}
L.~Xu and M.~Jordan, ``On convergence properties of the {EM} algorithm for
  {G}aussian mixtures,'' \emph{Neural Computation}, vol.~8, pp. 129--151, 1996.

\bibitem{chaudhuri09}
K.~Chaudhuri, S.~Dasgupta, and A.~Vattani, ``Learning mixtures of {G}aussians
  using the k-means algorithm,'' \emph{CoRR abs/0912.0086}, 2009.

\bibitem{Xu16}
J.~Xu, D.~Hsu, and A.~Maleki, ``Global analysis of expectation maximization for
  mixtures of two {G}aussians,'' in \emph{International Conference on Neural
  Information Processing Systems}, Barcelona, Spain, December 2016, pp.
  2684--2692.

\bibitem{Hardt-Price15}
M.~Hardt and E.~Price, ``Tight bounds for learning a mixture of two
  {G}aussians,'' in \emph{Symposium on Theory of Computing}, Portland, Oregon,
  USA, June 2015, pp. 753--760.

\bibitem{Pearson1894}
K.~Pearson, ``Contributions to the mathematical theory of evolution,''
  \emph{Proceedings of royal society of London}, vol. 185, pp. 71--110, 1894.

\bibitem{Wainwright17}
S.~Balakrishnan, M.~J. Wainwright, and B.~Yu, ``Statistical guarantees for the
  {EM} algorithm: {F}rom population to sample-based analysis,'' \emph{The
  Annals of Statistics}, vol.~45, no.~1, pp. 77--120, 2017.

\end{thebibliography}

\end{document}